\pdfoutput=1

\documentclass[a4paper,envcountsame]{llncs}

\usepackage[english]{babel}
\usepackage[utf8]{inputenc}
\usepackage[unicode]{hyperref}

\usepackage{amsmath}
\usepackage{amssymb}
\usepackage{bbm}
\usepackage{enumerate}
\usepackage{paralist} 
\usepackage{longtable} 
\usepackage{float} 
\usepackage{multicol}  

\usepackage[flowcharts]{aixi}

\usepackage{listings}
\lstdefinelanguage{pseudo}{
	morekeywords={program, if, then, else, while, for, do, not, return, output, true, false},
	sensitive=false,
	morecomment=[l]{//},
	morecomment=[s]{/*}{*/},
	morestring=[b]",
}
\def\BayesExp{{\mbox{BayesExp}}} 
\def\one{\mathbbm{1}} 

\AtBeginDocument{

}

\title{On the Computability of Solomonoff Induction and Knowledge-Seeking%
\thanks{The final publication is available at \url{http://link.springer.com/}.}}
\author{Jan Leike \and Marcus Hutter}

\titlerunning{On the Computability of Solomonoff Induction and Knowledge-Seeking}
\authorrunning{J.\ Leike and M.\ Hutter}
\institute{
  Australian National University \\
  \texttt{\{jan.leike|marcus.hutter\}@anu.edu.au}
}

\begin{document}

\maketitle

\begin{abstract}%
Solomonoff induction is held as a gold standard for learning,
but it is known to be incomputable.
We quantify its incomputability by placing
various flavors of Solomonoff's prior $M$ in the arithmetical hierarchy.
We also derive computability bounds for knowledge-seeking agents,
and give a limit-computable
weakly asymptotically optimal reinforcement learning agent.
\end{abstract}


\section{Introduction}
\label{sec:introduction}

Solomonoff's theory of learning~\cite{Solomonoff:1964,Solomonoff:1978,LV:2008},
commonly called \emph{So\-lo\-mo\-noff induction},
arguably solves the induction problem~\cite{RH:2011}:
for data drawn from any computable measure $\mu$,
Solomonoff induction will converge to the correct belief
about any hypothesis~\cite{BD:1962}.
Moreover, convergence is extremely fast in the sense that
the expected number of prediction errors is
$E + O(\sqrt{E})$ compared to
the number of errors $E$ made by
the informed predictor that knows $\mu$~\cite{Hutter:2001error}.

In \emph{reinforcement learning}
an agent repeatedly takes actions and receives observations and rewards.
The goal is to maximize cumulative (discounted) reward.
Solomonoff's ideas can be extended to reinforcement learning,
leading to the Bayesian agent AIXI~\cite{Hutter:2000,Hutter:2005}.
However, AIXI's trade-off between exploration and exploitation
includes insufficient exploration
to get rid of the prior's bias~\cite{LH:2015priors},
which is why the universal agent AIXI does not achieve
asymptotic optimality~\cite{Orseau:2010,Orseau:2013}.

For extra exploration, we can resort to
Orseau's \emph{knowledge-seeking agents}.
Instead of rewards,
knowledge-seeking agents maximize
entropy gain~\cite{Orseau:2011ksa,Orseau:2014ksa} or
expected information gain~\cite{OLH:2013ksa}.
These agents are apt explorers,
and asymptotically they learn their environment perfectly%
~\cite{Orseau:2014ksa,OLH:2013ksa}.

A reinforcement learning agent is \emph{weakly asymptotically optimal} if
the value of its policy
converges to the optimal value in Cesàro mean~\cite{LH:2011opt}.
Weak asymptotic optimality stands out because
it currently is the only known nontrivial objective notion of optimality
for general reinforcement learners~\cite{LH:2011opt,Orseau:2013,LH:2015priors}.
Lattimore defines the agent {\BayesExp}
by grafting a knowledge-seeking component on top of AIXI and
shows that
{\BayesExp} is a weakly asymptotically optimal agent in the class of all
stochastically computable environments~\cite[Ch.\ 5]{Lattimore:2013}.

The purpose of models such as
Solomonoff induction, AIXI, and knowledge-seeking agents
is to answer the question of how to solve (reinforcement) learning
\emph{in theory}.
These answers are useless if they cannot be approximated in practice,
i.e., by a regular Turing machine.
Therefore we posit that any ideal model
must at least be \emph{limit computable} ($\Delta^0_2$).

Limit computable functions are the functions that
admit an \emph{anytime algorithm}.
More generally,
the \emph{arithmetical hierarchy} specifies different levels of computability
based on \emph{oracle machines}:
each level in the arithmetical hierarchy
is computed by a Turing machine which
may query a halting oracle for the respective lower level.

\begin{table}[t]
\begin{center}
\renewcommand{\arraystretch}{1.2}
\setlength{\tabcolsep}{5pt}
\begin{tabular}{lcc}
$P$        & $\{ (x, q) \in \X^* \times \mathbb{Q} \mid P(x) > q \}$
           & $\{ (x, y, q) \in \X^* \times \X^* \times \mathbb{Q} \mid P(xy \mid x) > q \}$ \\
\hline
$M$        & $\Sigma^0_1 \setminus \Delta^0_1$
           & $\Delta^0_2 \setminus (\Sigma^0_1 \cup \Pi^0_1)$ \\
$M\norm$   & $\Delta^0_2 \setminus (\Sigma^0_1 \cup \Pi^0_1)$
           & $\Delta^0_2  \setminus (\Sigma^0_1 \cup \Pi^0_1)$ \\
$\MM$      & $\Pi^0_2 \setminus \Delta^0_2$
           & $\Delta^0_3 \setminus (\Sigma^0_2 \cup \Pi^0_2)$ \\
$\MM\norm$ & $\Delta^0_3 \setminus (\Sigma^0_2 \cup \Pi^0_2)$
           & $\Delta^0_3 \setminus (\Sigma^0_2 \cup \Pi^0_2)$
\end{tabular}
\end{center}
\caption{
The computability results on $M$, $M\norm$, $\MM$, and $\MM\norm$
proved in \autoref{sec:complexity-induction}.
Lower bounds on the complexity of $\MM$ and $\MM\norm$
are given only for specific universal Turing machines.
}
\label{tab:complexity-induction}
\end{table}

In previous work~\cite{LH:2015computability} we established that
AIXI is limit computable if restricted to $\varepsilon$-optimal policies,
and placed various versions of AIXI, AINU, and AIMU
in the arithmetical hierarchy.
In this paper we investigate
the (in-)com\-putability of Solomonoff induction and knowledge-seeking.
The universal prior $M$ is lower semicomputable and
hence its conditional is limit computable.
But $M$ is a semimeasure: it assigns positive probability that
the observed string has only finite length.
This can be circumvented by normalizing $M$.
Solomonoff's normalization $M\norm$
preserves the ratio $M(x1) / M(x0)$ and is limit computable.
If we remove the contribution of programs that compute only finite strings,
we get a semimeasure $\MM$, which can be normalized to $\MM\norm$
by multiplication with a constant.
We show that both $\MM$ and $\MM\norm$
are \emph{not} limit computable.
Our results on the computability of Solomonoff induction are stated in
\autoref{tab:complexity-induction}
and proved in \autoref{sec:complexity-induction}.
In \autoref{sec:complexity-knowledge-seeking}
we show that for finite horizons
both the entropy-seeking and the information-seeking
agent are $\Delta^0_3$-computable and have
limit-computable $\varepsilon$-optimal policies.
The weakly asymptotically optimal agent
{\BayesExp} relies on optimal policies
that are generally not limit computable~\cite[Thm.\ 16]{LH:2015computability}.
In \autoref{sec:computability-wao}
we give a weakly asymptotically optimal agent
based on {\BayesExp} that is limit computable.
A list of notation can be found on
\hyperref[app:notation]{page~\pageref*{app:notation}}.

\section{Preliminaries}
\label{sec:preliminaries}

We use the setup and notation from \cite{LH:2015computability}.

\subsection{The Arithmetical Hierarchy}

A set $A \subseteq \mathbb{N}$ \emph{is $\Sigma^0_n$} iff
there is a computable relation $S$ such that
\begin{equation}\label{eq:def-Sigma^0_n}
k \in A
\;\Longleftrightarrow\;
\exists k_1 \forall k_2 \ldots Q_n k_n\; S(k, k_1, \ldots, k_n)
\end{equation}
where $Q_n = \forall$ if $n$ is even, $Q_n = \exists$ if $n$ is odd%
~\cite[Def.\ 1.4.10]{Nies:2009}.
A set $A \subseteq \mathbb{N}$ \emph{is $\Pi^0_n$} iff
its complement $\mathbb{N} \setminus A$ is $\Sigma^0_n$.
We call the formula on the right hand side of \eqref{eq:def-Sigma^0_n} a
\emph{$\Sigma^0_n$-formula}, its negation is called \emph{$\Pi^0_n$-formula}.
It can be shown that
we can add any bounded quantifiers and
duplicate quantifiers of the same type
without changing the classification of $A$.
The set $A$ \emph{is $\Delta^0_n$} iff $A$ is $\Sigma^0_n$ and $A$ is $\Pi^0_n$.
We get that
$\Sigma^0_1$ as the class of recursively enumerable sets,
$\Pi^0_1$ as the class of co-recursively enumerable sets and
$\Delta^0_1$ as the class of recursive sets.

We say the set $A \subseteq \mathbb{N}$ is \emph{$\Sigma^0_n$-hard
($\Pi^0_n$-hard, $\Delta^0_n$-hard)} iff
for any set $B \in \Sigma^0_n$ ($B \in \Pi^0_n$, $B \in \Delta^0_n$),
$B$ is many-one reducible to $A$, i.e.,
there is a computable function $f$ such that
$k \in B \leftrightarrow f(k) \in A$~\cite[Def.\ 1.2.1]{Nies:2009}.
We get $\Sigma^0_n \subset \Delta^0_{n+1} \subset \Sigma^0_{n+1} \subset \ldots$
and $\Pi^0_n \subset \Delta^0_{n+1} \subset \Pi^0_{n+1} \subset \ldots$.
This hierarchy of subsets of natural numbers is known as
the \emph{arithmetical hierarchy}.

By Post's Theorem~\cite[Thm.\ 1.4.13]{Nies:2009},
a set is $\Sigma^0_n$ if and only if
it is recursively enumerable on an oracle machine
with an oracle for a $\Sigma^0_{n-1}$-hard set.

\subsection{Strings}

Let $\X$ be some finite set called \emph{alphabet}.
The set $\X^* := \bigcup_{n=0}^\infty \X^n$ is
the set of all finite strings over the alphabet $\X$,
the set $\X^\infty$ is
the set of all infinite strings
over the alphabet $\X$, and
the set $\X^\sharp := \X^* \cup \X^\infty$ is their union.
The empty string is denoted by $\epsilon$, not to be confused
with the small positive real number $\varepsilon$.
Given a string $x \in \X^*$, we denote its length by $|x|$.
For a (finite or infinite) string $x$ of length $\geq k$,
we denote with $x_{1:k}$ the first $k$ characters of $x$,
and with $x_{<k}$ the first $k - 1$ characters of $x$.
The notation $x_{1:\infty}$ stresses that $x$ is an infinite string.
We write $x \sqsubseteq y$ iff $x$ is a prefix of $y$, i.e., $x = y_{1:|x|}$.

\subsection{Computability of Real-valued Functions}

We fix some encoding of rational numbers into binary strings and
an encoding of binary strings into natural numbers.
From now on, this encoding will be done implicitly wherever necessary.

\begin{definition}[$\Sigma^0_n$-, $\Pi^0_n$-, $\Delta^0_n$-computable]
\label{def:computable}
A function $f: \X^* \to \mathbb{R}$ is called
\emph{$\Sigma^0_n$-computable ($\Pi^0_n$-computable, $\Delta^0_n$-computable)} iff
the set $\{ (x, q) \in \X^* \times \mathbb{Q} \mid f(x) > q \}$
is $\Sigma^0_n$ ($\Pi^0_n$, $\Delta^0_n$).
\end{definition}
A $\Delta^0_1$-computable function is called \emph{computable},
a $\Sigma^0_1$-computable function is called \emph{lower semicomputable}, and
a $\Pi^0_1$-computable function is called \emph{upper semicomputable}.
A $\Delta^0_2$-computable function $f$ is called \emph{limit computable},
because there is a computable function $\phi$ such that
\[
\lim_{k \to \infty} \phi(x, k) = f(x).
\]
The program $\phi$ that limit computes $f$
can be thought of as an \emph{anytime algorithm} for $f$:
we can stop $\phi$ at any time $k$ and get a preliminary answer.
If the program $\phi$ ran long enough (which we do not know),
this preliminary answer will be close to the correct one.

Limit-computable sets are the highest level in the arithmetical hierarchy
that can be approached by a regular Turing machine.
Above limit-computable sets
we necessarily need some form of halting oracle.
See \autoref{tab:computability} for the definition of
lower/upper semicomputable and limit-computable functions
in terms of the arithmetical hierarchy.

\begin{table}[t]
\begin{center}
\setlength{\tabcolsep}{2mm} 
\begin{tabular}{lcc}
& $\{ (x, q) \mid f(x) > q \}$
& $\{ (x, q) \mid f(x) < q \}$ \\
\hline
$f$ is computable            & $\Delta^0_1$ & $\Delta^0_1$ \\
$f$ is lower semicomputable  & $\Sigma^0_1$ & $\Pi^0_1$    \\
$f$ is upper semicomputable  & $\Pi^0_1$    & $\Sigma^0_1$ \\
$f$ is limit computable      & $\Delta^0_2$ & $\Delta^0_2$ \\
$f$ is $\Delta^0_n$-computable & $\Delta^0_n$ & $\Delta^0_n$ \\
$f$ is $\Sigma^0_n$-computable & $\Sigma^0_n$ & $\Pi^0_n$    \\
$f$ is $\Pi^0_n$-computable    & $\Pi^0_n$    & $\Sigma^0_n$ \\
\end{tabular}
\end{center}
\caption{
Connection between the computability of real-valued functions and
the arithmetical hierarchy.
}
\label{tab:computability}
\end{table}

\begin{lemma}[Computability of Arithmetical Operations]
\label{lem:computable-reals}
Let $n > 0$ and
let $f, g: \X^* \to \mathbb{R}$ be two $\Delta^0_n$-computable functions.
Then
\begin{enumerate}[(i)]
\item $\{ (x, y) \mid f(x) >    g(y) \}$ is $\Sigma^0_n$,
\item $\{ (x, y) \mid f(x) \leq g(y) \}$ is $\Pi^0_n$,
\item $f + g$, $f - g$, and $f \cdot g$ are $\Delta^0_n$-computable,
\item $f / g$ is $\Delta^0_n$-computable if $g(x) \neq 0$ for all $x$, and
\item $\log f$ is $\Delta^0_n$-computable if $f(x) > 0$ for all $x$.
\end{enumerate}
\end{lemma}

\section{The Complexity of Solomonoff Induction}
\label{sec:complexity-induction}

A \emph{semimeasure} over the alphabet $\X$ is
a function $\nu: \X^* \to [0,1]$ such that
\begin{inparaenum}[(i)]
\item $\nu(\epsilon) \leq 1$, and
\item $\nu(x) \geq \sum_{a \in \X} \nu(xa)$ for all $x \in \X^*$.
\end{inparaenum}
A semimeasure is called (probability) \emph{measure} iff
for all $x$ equalities hold in (i) and (ii).

\emph{Solomonoff's prior $M$}~\cite{Solomonoff:1964} assigns to a string $x$
the probability that
the reference universal monotone Turing machine $U$~\cite[Ch.\ 4.5.2]{LV:2008}
computes a string starting with $x$
when fed with uniformly random bits as input.
Formally,
\begin{equation}\label{eq:def-M}
M(x) := \sum_{p:\, x \sqsubseteq U(p)} 2^{-|p|}.
\end{equation}
The function $M$ is a lower semicomputable semimeasure,
but not computable and not a measure~\cite[Lem.\ 4.5.3]{LV:2008}.
A semimeasure $\nu$ can be turned into a measure $\nu\norm$
using \emph{Solomonoff normalization}:
$\nu\norm(\epsilon) := 1$ and
for all $x \in \X^*$ and $a \in \X$,
\begin{equation}\label{eq:normalization}
   \nu\norm(xa)
:= \nu\norm(x) \frac{\nu(xa)}{\sum_{b \in \X} \nu(xb)}.
\end{equation}
By definition, $M\norm$ and $\MM\norm$ are measures~\cite[Sec.\ 4.5.3]{LV:2008}.
Moreover, since $M\norm \geq M$, normalization preserves universal dominance.
Hence Solomonoff's theorem implies that
$M\norm$ predicts just as well as $M$.

The \emph{measure mixture $\MM$}~\cite[p.\ 74]{Gacs:1983}
is defined as
\begin{equation}\label{eq:def-MM}
\MM(x) := \lim_{n \to \infty} \sum_{y \in \X^n} M(xy).
\end{equation}
The measure mixture $\MM$ is the same as $M$ except that
the contributions by programs that do not produce infinite strings are removed:
for any such program $p$,
let $k$ denote the length of the finite string generated by $p$.
Then for $|xy| > k$,
the program $p$ does not contribute to $M(xy)$,
hence it is excluded from $\MM(x)$.

Similarly to $M$,
the measure mixture $\MM$ is not a (probability) measure
since $\MM(\varepsilon) < 1$,
but in this case normalization \eqref{eq:normalization}
is just multiplication with the constant $1/\MM(\epsilon)$,
leading to the \emph{normalized measure mixture} $\MM\norm$.
When using the Solomonoff prior $M$
(or one of its sisters $M\norm$, $\MM$, or $\MM\norm$)
for sequence prediction,
we need to compute the conditional probability
$M(xy \mid x) := M(xy) / M(x)$
for finite strings $x, y \in \X^*$.
Because $M(x) > 0$ for all finite strings $x \in \X^*$,
this quotient is well-defined.

\begin{figure}[t]
\begin{align*}
M(xy \mid x) > q
~~\Longleftrightarrow~~
\forall \ell \exists k\; \frac{\phi(xy, k)}{\phi(x, \ell)} > q
~~\Longleftrightarrow~~
\exists k \exists \ell_0 \forall \ell \geq \ell_0\;
	\frac{\phi(xy, k)}{\phi(x, \ell)} > q
\end{align*}
\caption{
A $\Pi^0_2$-formula and an equivalent $\Sigma^0_2$-formula
defining conditional $M$.
Here $\phi(x, k)$ denotes a computable function that lower semicomputes $M(x)$.
}\label{fig:conditional-M}
\end{figure}

\begin{theorem}[Complexity of $M$, $M\norm$, $\MM$, and $\MM\norm$]
\label{thm:complexity-M}
\nopagebreak
\vspace{-1em}
\begin{multicols}{2}
\begin{enumerate}[(i)]
\item $M(x)$ is lower semicomputable
\item $M(xy \mid x)$ is limit computable
\item $M\norm(x)$ is limit computable
\item $M\norm(xy \mid x)$ is limit computable
\item $\MM(x)$ is $\Pi^0_2$-computable
\item $\MM(xy \mid x)$ is $\Delta^0_3$-computable
\item $\MM\norm(x)$ is $\Delta^0_3$-computable
\item $\MM\norm(xy \mid x)$ is $\Delta^0_3$-computable
\end{enumerate}
\end{multicols}
\end{theorem}
\begin{proof}
\begin{enumerate}[(i)]
\item By \cite[Thm.\ 4.5.2]{LV:2008}.
	Intuitively, we can run all programs in parallel and get
	monotonely increasing lower bounds for $M(x)$ by adding $2^{-|p|}$
	every time a program $p$ has completed outputting $x$.
\item From (i) and \autoref{lem:computable-reals} (iv),
	since $M(x) > 0$ (see also \autoref{fig:conditional-M}).
\item By \autoref{lem:computable-reals} (iii,iv) and $M(x) > 0$.
\item By (iii) and \autoref{lem:computable-reals} (iv),
	since $M\norm(x) \geq M(x) > 0$.
\item Let $\phi$ be a computable function that lower semicomputes $M$.
Since $M$ is a semimeasure, $M(xy) \geq \sum_z M(xyz)$,
hence $\sum_{y \in \X^n} M(xy)$ is nonincreasing in $n$ and thus
$\MM(x) > q$ iff
$\forall n \exists k \sum_{y \in \X^n} \phi(xy, k) > q$.
\item From (v) and \autoref{lem:computable-reals} (iv),
	since $\MM(x) > 0 $. 
\item From (v) and \autoref{lem:computable-reals} (iv).
\item From (vi) and \autoref{lem:computable-reals} (iv),
	since $\MM\norm(x) \geq \MM(x) > 0$.
\qed
\end{enumerate}
\end{proof}

We proceed to show that these bounds are in fact the best possible ones.
If $M$ were $\Delta^0_1$-computable,
then so would be the conditional semimeasure $M(\,\cdot \mid \cdot\,)$.
Thus we could compute the $M$-adversarial sequence $z_1 z_2 \ldots$ defined by
\[
z_t :=
\begin{cases}
0 &\text{if } M(1 \mid z_{<t}) > \tfrac{1}{2}, \\
1 &\text{otherwise}.
\end{cases}
\]
The sequence $z_1 z_2 \ldots$ corresponds to
a computable deterministic measure $\mu$.
However, we have $M(z_{1:t}) \leq 2^{-t}$ by construction,
so dominance $M(x) \geq w_\mu \mu(x)$ with $w_\mu > 0$
yields a contradiction with $t \to \infty$:
\[
2^{-t} \geq M(z_{1:t}) \geq w_\mu \mu(z_{1:t}) = w_\mu > 0
\]
By the same argument,
the normalized Solomonoff prior $M\norm$ cannot be $\Delta^0_1$-computable.
However, since it is a measure,
$\Sigma^0_1$- or $\Pi^0_1$-computability would entail $\Delta^0_1$-computability.

For $\MM$ and $\MM\norm$ we prove the following two lower bounds
for specific universal Turing machines.

\begin{theorem}[$\MM$ is not Limit Computable]
\label{thm:MM-is-not-Delta2}
There is a universal Turing machine $U'$ such that
the set $\{ (x, q) \mid \MM_{U'}(x) > q \}$ is not in $\Delta^0_2$.
\end{theorem}
\begin{proof}
Assume the contrary,
let $A$ be $\Pi^0_2$ but not $\Delta^0_2$, and
let $S$ be a computable relation such that
\begin{equation}
\label{eq:eta}
n \in A
\quad\Longleftrightarrow\quad
\forall k \exists i\; S(n, k, i).
\end{equation}
For each $n \in \mathbb{N}$,
we define the program $p_n$ as follows.
\begin{center}
\begin{minipage}{60mm}
\begin{lstlisting}
	output $1^{n+1} 0$
	$k$ := $0$
	while true:
		$i$ := $0$
		while not $S(n, k, i)$:
			$i$ := $i + 1$
		$k$ := $k + 1$
		output $0$
\end{lstlisting}
\end{minipage}
\end{center}
Each program $p_n$ always outputs $1^{n+1}0$.
Furthermore, the program $p_n$ outputs the infinite string $1^{n+1}0^\infty$
if and only if $n \in A$ by \eqref{eq:eta}.
We define $U'$ as follows using our reference machine $U$.
\begin{itemize}
\item $U'(1^{n+1}0)$: Run $p_n$.
\item $U'(00p)$: Run $U(p)$.
\item $U'(01p)$: Run $U(p)$ and bitwise invert its output.
\end{itemize}
By construction, $U'$ is a universal Turing machine.
No $p_n$ outputs a string starting with $0^{n+1} 1$,
therefore $\MM_{U'}(0^{n+1}1) = \tfrac{1}{4} \big( \MM_U(0^{n+1}1) + \MM_U(1^{n+1} 0) \big)$.
Hence
\begin{align*}
   \MM_{U'}(1^{n+1} 0)
&= 2^{-n-2} \one_{A}(n)
     + \tfrac{1}{4} \MM_U(1^{n+1}0) + \tfrac{1}{4} \MM_U(0^{n+1}1) \\
&= 2^{-n-2} \one_{A}(n) + \MM_{U'}(0^{n+1}1)
\end{align*}
If $n \notin A$, then $\MM_{U'}(1^{n+1}0) = \MM_{U'}(0^{n+1}1)$.
Otherwise, we have
$|\MM_{U'}(1^{n+1}0) - \MM_{U'}(0^{n+1}1)| = 2^{-n-2}$.

Now we assume that $\MM_{U'}$ is limit computable, i.e.,
there is a computable function $\phi: \X^* \times \mathbb{N} \to \mathbb{Q}$
such that
$\lim_{k \to \infty} \phi(x, k) = \MM_{U'}(x)$.
We get that
\[
n \in A
\;\Longleftrightarrow\;
\lim_{k \to \infty} \phi(0^{n+1}1, k) - \phi(1^{n+1}0, k) > 2^{-n-3},
\]
thus $A$ is limit computable, a contradiction.
\qed
\end{proof}

\begin{corollary}[$\MM\norm$ is not $\Sigma^0_2$- or $\Pi^0_2$-computable]
\label{cor:MMnorm-is-not-2}
There is a universal Turing machine $U'$ such that
$\{ (x, q) \mid {\MM\norm}_{U'}(x) > q \}$ is not in $\Sigma^0_2$ or $\Pi^0_2$.
\end{corollary}
\begin{proof}
Since $\MM\norm = c \cdot \MM$,
there exists a $k \in \mathbb{N}$ such that $2^{-k} < c$
(even if we do not know the value of $k$).
We can show that
the set $\{ (x, q) \mid {\MM\norm}_{U'}(x) > q \}$ is not in $\Delta^0_2$
analogously to the proof of \autoref{thm:MM-is-not-Delta2},
using
\[
n \in A
\;\Longleftrightarrow\;
\lim_{k \to \infty} \phi(0^{n+1}1, k) - \phi(1^{n+1}0, k) > 2^{-k-n-3}.
\]
If $\MM\norm$ were $\Sigma^0_2$-computable or $\Pi^0_2$-computable,
this would imply that $\MM\norm$ is $\Delta^0_2$-computable
since $\MM\norm$ is a measure, a contradiction.
\qed
\end{proof}

Since $M(\epsilon) = 1$,
we have $M(x \mid \epsilon) = M(x)$,
so the conditional probability $M(xy \mid x)$ has at least the same complexity
as $M$.
Analogously for $M\norm$ and $\MM\norm$ since they are measures.
For $\MM$, we have that $\MM(x \mid \epsilon) = \MM\norm(x)$,
so \autoref{cor:MMnorm-is-not-2} applies.
All that remains to prove is
that conditional $M$ is not lower semicomputable.

\begin{theorem}[Conditional $M$ is not Lower Semicomputable]
\label{thm:M-conditional-is-not-Sigma1}
The set $\{ (x, xy, q) \mid M(xy \mid x) > q \}$ is not recursively enumerable.
\end{theorem}


\begin{proof}
Assume to the contrary that $M(xy \mid x)$ is lower semicomputable.
According to \cite[Thm.\ 12]{LHG:2011evenbits}
there is an infinite string $z_{1:\infty}$ such that
$z_{2t} = z_{2t-1}$ for all $t > 0$ and
\begin{equation}\label{eq:M-does-not-converge-to-1}
\liminf_{t \to \infty} M(z_{1:2t} \mid z_{<2t}) < 1.
\end{equation}
Define the semimeasure
\[
   \nu(x_{1:t})
:=
\begin{cases}
\prod_{k=1}^{\lceil t/2 \rceil} M(x_{<2k} \mid x_{<2k-1})
  & \text{if } \forall 0 < 2k \leq t\; x_{2k} = x_{2k-1} \\
0 & \text{otherwise.}
\end{cases}
\]
Since we assume $M(x_{<2k} \mid x_{<2k-1})$ to be lower semicomputable,
$\nu$ is lower semicomputable.
Therefore there is a constant $c > 0$ such that
$M(x) \geq c \nu(x)$ for all $x \in \X^*$.
With the chain rule we get for even-lengthed $x$
with $x_{2k} = x_{2k-1}$
\[
     c
\leq \frac{M(x)}{\nu(x)}
=    \frac{\prod_{i=1}^{t} M(x_{1:i} \mid x_{<i})}%
          {\prod_{k=1}^{t/2} M(x_{<2k} \mid x_{<2k-1})}
=    \prod_{k=1}^{t/2} M(x_{1:2k} \mid x_{<2k}).
\]
Plugging in the sequence $z_{1:\infty}$,
we get a contradiction with \eqref{eq:M-does-not-converge-to-1}:
\[
     0
<    c
\leq \prod_{k=1}^t M(z_{1:2k} \mid z_{<2k})
\xrightarrow{t \to \infty}
     0
\eqno\qed
\]
\end{proof}

\section{The Complexity of Knowledge-Seeking}
\label{sec:complexity-knowledge-seeking}

In general reinforcement learning
the agent interacts with an environment in cycles:
at time step $t$ the agent chooses an \emph{action} $a_t \in \A$ and
receives a \emph{percept} $e_t = (o_t, r_t) \in \E$
consisting of an \emph{observation} $o_t \in \O$
and a real-valued \emph{reward} $r_t \in \mathbb{R}$;
the cycle then repeats for $t + 1$.
A \emph{history} is an element of $\H$.
We use $\ae \in \A \times \E$ to denote one interaction cycle,
and $\ae_{1:t}$ to denote a history of length $t$.
A \emph{policy} is a function $\pi: \H \to \A$
mapping each history to the action taken after seeing this history.
We assume $\A$ and $\E$ to be finite.

The environment can be stochastic,
but is assumed to be semicomputable.
In accordance with the AIXI literature~\cite{Hutter:2005},
we model environments as lower semicomputable
\emph{chronological conditional semimeasures} (LSCCCSs).
The class of of all LSCCCSs is denoted with $\M$.
A \emph{conditional semimeasure} $\nu$ takes a sequence of actions $a_{1:t}$ as input
and returns a semimeasure $\nu(\,\cdot \dmid a_{1:t})$ over $\E^\sharp$.
A conditional semimeasure $\nu$ is \emph{chronological} iff
percepts at time $t$ do not depend on future actions, i.e.,
$\nu(e_{1:t} \dmid a_{1:k}) = \nu(e_{1:t} \dmid a_{1:t})$ for all $k > t$.
Despite their name,
conditional semimeasures do \emph{not} specify conditional probabilities;
the environment $\nu$ is \emph{not}
a joint probability distribution on actions and percepts.
Here we only care about the computability of the environment $\nu$;
for our purposes,
chronological conditional semimeasures behave just like semimeasures.

Equivalently to \eqref{eq:def-M},
the Solomonoff prior $M$ can be defined as
a mixture over all lower semicomputable semimeasures
using a lower semicomputable \emph{universal prior}~\cite{WSH:2011}.
We generalize this representation to
chronological conditional semimeasures:
we fix the lower semicomputable universal prior $(w_\nu)_{\nu \in \M}$ with
$w_\nu > 0$ for all $\nu \in \M$ and $\sum_{\nu \in \M} w_\nu \leq 1$,
given by the reference machine $U$
according to $w_\nu := 2^{-K_U(\nu)}$~\cite[Sec.\ 5.1.2]{Hutter:2005}.
The universal prior $w$ gives rise to
the \emph{universal mixture $\xi$},
which is a convex combination of all LSCCCSs $\M$:
\[
\xi(e_{<t} \dmid a_{<t}) := \sum_{\nu \in \M} w_\nu \nu(e_{<t} \dmid a_{<t})
\]
The universal mixture $\xi$
is analogous to the Solomonoff prior $M$ but defined for reactive environments.
Analogously to \autoref{thm:complexity-M} (i),
the universal mixture $\xi$
is lower semicomputable~\cite[Sec.\ 5.10]{Hutter:2005}.
Moreover, we have $\xi\norm \geq \xi$,
preserving universal dominance analogously to $M$.

\subsection{Knowledge-Seeking Agents}
\label{ssec:knowledge-seeking-agents}

We discuss two variants of knowledge-seeking agents:
entropy-seeking agents (Shannon-KSA)~\cite{Orseau:2011ksa,Orseau:2014ksa}
and information-seeking agents (KL-KSA)~\cite{OLH:2013ksa}.
The en\-tro\-py-seeking agent maximizes the Shannon entropy gain,
while the information-seeking agent maximizes
the expected Bayesian information gain (KL-divergence)
in the universal mixture $\xi$.
These quantities are expressed in the \emph{value function}.

In this section
we use a finite lifetime $m$
(possibly dependent on time step $t$):
the knowledge-seeking agent maximizes entropy/information received up
to and including time step $m$.
We assume that the function $m$ (of $t$) is computable.

\begin{definition}[{Entropy-Seeking Value Function~\cite[Sec.\ 6]{Orseau:2014ksa}}]
\label{def:V-entropy}
The \emph{en\-tro\-py-seeking value} of a policy $\pi$ given history $\ae_{<t}$ is
\[
   V^\pi_H(\ae_{<t})
:= \sum_{e_{t:m}} -\xi\norm(e_{1:m} \mid e_{<t} \dmid a_{1:m})
     \log_2 \xi\norm(e_{1:m} \mid e_{<t} \dmid a_{1:m})
\]
where $a_i := \pi(e_{<i})$ for all $i \geq t$.
\end{definition}

\begin{definition}[{Information-Seeking Value Function~\cite[Def.\ 1]{OLH:2013ksa}}]
\label{def:V-information}
The \emph{information-seeking value} of a policy $\pi$ given history $\ae_{<t}$
is
\[
   V^\pi_I(\ae_{<t})
:= \sum_{e_{t:m}} \sum_{\nu \in \M} w_\nu
     \frac{\nu(e_{1:m} \dmid a_{1:m})}{\xi\norm(e_{<t} \dmid a_{<t})}
     \log_2 \frac{\nu(e_{1:m} \mid e_{<t} \dmid a_{1:m})}%
       {\xi\norm(e_{1:m} \mid e_{<t} \dmid a_{1:m})}
\]
where $a_i := \pi(e_{<i})$ for all $i \geq t$.
\end{definition}

We use $V^\pi$ in places where either of the entropy-seeking or
the information-seeking value function can be substituted.

\begin{definition}[($\varepsilon$-)Optimal Policy]
\label{def:optimal-policy}
The \emph{optimal value function} $V^*$ is defined as
$V^*(\ae_{<t}) := \sup_\pi V^\pi(\ae_{<t})$.
A policy $\pi$ is \emph{optimal} iff
$V^\pi(\ae_{<t}) = V^*(\ae_{<t})$
for all histories $\ae_{<t} \in \H$.
A policy $\pi$ is \emph{$\varepsilon$-optimal} iff
$V^*(\ae_{<t}) - V^\pi(\ae_{<t}) < \varepsilon$
for all histories $\ae_{<t} \in \H$.
\end{definition}

An entropy-seeking agent is defined as
an optimal policy for the value function $V_H^*$ and
an information-seeking agent is defined as
an optimal policy for the value function $V_I^*$.

The entropy-seeking agent does not work well in stochastic environments
because it gets distracted by noise in the environment
rather than trying to distinguish environments~\cite{OLH:2013ksa}.
Moreover, the unnormalized knowledge-seeking agents
may fail to seek knowledge
in deterministic semimeasures
as the following example demonstrates.

\begin{example}[Unnormalized Entropy-Seeking]
\label{ex:ksa-unnormalized}
Suppose we use $\xi$ instead of $\xi\norm$ in \autoref{def:V-entropy}.
Fix $\A := \{ \alpha, \beta \}$, $\E := \{ 0, 1 \}$, and $m := 1$
(we only care about the entropy of the next percept).
We illustrate the problem on a simple class of environments
$\{ \nu_1, \nu_2 \}$:
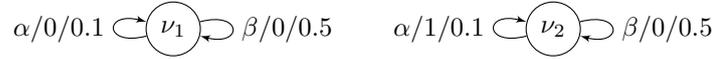
\begin{center}
\begin{tikzpicture}[auto]
\node[circle,draw] (nu1) {$\nu_1$};
\path[transition] (nu1) to[loop left] node {$\alpha/0/0.1$} (nu1);
\path[transition] (nu1) to[loop right] node {$\beta/0/0.5$} (nu1);

\node[circle,draw,right of=nu1,node distance=50mm] (nu2) {$\nu_2$};
\path[transition] (nu2) to[loop left] node {$\alpha/1/0.1$} (nu2);
\path[transition] (nu2) to[loop right] node {$\beta/0/0.5$} (nu2);
\end{tikzpicture}
\end{center}
where transitions are labeled with action/percept/probability.
Both $\nu_1$ and $\nu_2$ return a percept deterministically or nothing at all
(the environment ends).
Only action $\alpha$ distinguishes between the environments.
With the prior $w_{\nu_1} := w_{\nu_2} := 1/2$, we get a mixture $\xi$
for the entropy-seeking value function $V^\pi_H$.
Then $V^*_H(\alpha) \approx 0.432 < 0.5 = V^*_H(\beta)$, 
hence action $\beta$ is preferred over $\alpha$ by the entropy-seeking agent.
But taking action $\beta$ yields percept $0$ (if any),
hence nothing is learned about the environment.
\hfill$\Diamond$
\end{example}

Solomonoff's prior is extremely good at learning:
with this prior a Bayesian agent learns the value of its own policy
asymptotically (on-policy value convergence)~\cite[Thm.\ 5.36]{Hutter:2005}.
However, generally it does not learn
the result of counterfactual actions that it does not take.
Knowledge-seeking agents learn the environment more effectively,
because they focus on exploration.
Both the entropy-seeking agent and the information-seeking agent are
\emph{strongly asymptotically optimal}
in the class of all deterministic computable environments%
~\cite[Thm.\ 5]{Orseau:2014ksa,OLH:2013ksa}:
the value of their policy converges to the optimal value in the sense that
$V^\pi \to V^*$ almost surely.
Moreover, the information-seeking agent also learns to predict the result of
counterfactual actions~\cite[Thm.\ 7]{OLH:2013ksa}.

\subsection{Knowledge-Seeking is Limit Computable}
\label{ssec:computability-knowledge-seeking}

We proceed to show that $\varepsilon$-optimal
knowledge-seeking agents are limit computable, and
optimal knowledge-seeking agents are in $\Delta^0_3$.

\begin{theorem}[Computability of Knowledge-Seeking]
\label{thm:limit-computable-knowledge-seeking}
There are limit-computable $\varepsilon$-optimal policies and
$\Delta^0_3$-computable optimal policies for
entropy-seeking and information-seeking agents.
\end{theorem}
\begin{proof}
Since $\xi$, $\nu$, and $w_\nu$ are lower semicomputable,
the value functions $V^*_H$ and $V^*_I$ are $\Delta^0_2$-computable
according to \autoref{lem:computable-reals} (iii-v).
The claim now follows from
the following lemma.
\qed
\end{proof}

\begin{lemma}[{Complexity of ($\varepsilon$-)Optimal Policies%
~\cite[Thm.\ 8 \& 11]{LH:2015computability}}]
\label{lem:complexity-optimal-policies}
If the optimal value function $V^*$ is $\Delta^0_n$-computable,
then there is an optimal policy $\pi^*$
that is in $\Delta^0_{n+1}$, and
there is an $\varepsilon$-optimal policy $\pi^\varepsilon$
that is in $\Delta^0_n$.
\end{lemma}

\section{A Weakly Asymptotically Optimal Agent in \texorpdfstring{$\Delta^0_2$}{∆⁰₂}}
\label{sec:computability-wao}

In reinforcement learning
we are interested in \emph{reward-seeking} policies.
Rewards are provided by the environment
as part of each percept $e_t = (o_t, r_t)$ where
$o_t \in \O$ is the \emph{observation} and
$r_t \in [0, 1]$ is the \emph{reward}.
In this section we fix a computable discount function
$\gamma: \mathbb{N} \to \mathbb{R}$ with
$\gamma(t) \geq 0$ and $\sum_{t=1}^\infty \gamma(t) < \infty$.
The \emph{discount normalization factor} is defined as
$\Gamma_t := \sum_{i=t}^\infty \gamma(i)$.
The \emph{effective horizon} $H_t(\varepsilon)$ is a horizon
that is long enough to encompass all but an $\varepsilon$
of the discount function's mass:
\[
H_t(\varepsilon) := \min \{ k \mid \Gamma_{t+k} / \Gamma_t \leq \varepsilon \}.
\]

\begin{definition}[{Reward-Seeking Value Function%
~\cite[Def.\ 20]{LH:2015computability}}]
\label{def:V-reward}
The \emph{re\-ward-seeking value} of a policy $\pi$ in environment $\nu$
given history $\ae_{<t}$ is
\[
   V^\pi_\nu(\ae_{<t})
:= \frac{1}{\Gamma_t} \sum_{m=t}^\infty \sum_{e_{t:m}} \gamma(m) r_m
     \nu( e_{1:m} \mid e_{<t} \dmid a_{1:m})
\]
if $\Gamma_t > 0$ and $V^\pi_\nu(\ae_{<t}) := 0$ if $\Gamma_t = 0$
where $a_i := \pi(e_{<i})$ for all $i \geq t$.
\end{definition}

\begin{definition}[{Weak Asymptotic Optimality~\cite[Def.\ 7]{LH:2011opt}}]
\label{def:wao}
A policy $\pi$ is \emph{weakly asymptotically optimal}
in the class of environments $\M$ iff
the reward-seeking value converges to the optimal value on-policy
in Cesàro mean, i.e.,
\[
\frac{1}{t} \sum_{k=1}^t \big( V^*_\nu(\ae_{<k}) - V^\pi_\nu(\ae_{<k}) \big)
\xrightarrow{t \to \infty} 0
\quad\text{$\nu$-almost surely for all $\nu \in \M$}.
\]
\end{definition}
Not all discount functions admit weakly asymptotically optimal policies%
~\cite[Thm.\ 8]{LH:2011opt};
a necessary condition is
that the effective horizon grows sub\-linear\-ly%
~\cite[Thm.\ 5.5]{Lattimore:2013}.
This is satisfied by geometric discounting, but not by
harmonic or power discounting~\cite[Tab.\ 5.41]{Hutter:2005}.

This condition is also sufficient~\cite[Thm.\ 5.6]{Lattimore:2013}:
Lattimore defines a weakly asymptotically optimal agent called \emph{\BayesExp}%
~\cite[Ch.\ 5]{Lattimore:2013}.
{\BayesExp} alternates between phases of exploration and
phases of exploitation:
if the optimal in\-for\-ma\-tion-seeking value is larger than $\varepsilon_t$,
then {\BayesExp} starts an exploration phase,
otherwise it starts an exploitation phase.
During an exploration phase, {\BayesExp} follows
an optimal information-seeking policy for $H_t(\varepsilon_t)$ steps.
During an exploitation phase, {\BayesExp} follows
an $\xi$-optimal reward-seeking policy for one step~\cite[Alg.\ 2]{Lattimore:2013}.

Generally, optimal reward-seeking policies are $\Pi^0_2$-hard%
~\cite[Thm.\ 16]{LH:2015computability},
and for optimal knowledge-seeking policies
we only proved that they are $\Delta^0_3$.
Therefore we do not know {\BayesExp} to be limit computable,
and we expect it not to be.
However, we can approximate it using $\varepsilon$-optimal policies
preserving weak asymptotic optimality.

\begin{theorem}[A Limit-Computable Weakly Asymptotically Optimal Agent]
\label{thm:wao-limit-computable}
If there is a nonincreasing computable sequence of positive reals
$(\varepsilon_t)_{t \in \mathbb{N}}$ such that
$\varepsilon_t \to 0$ and
$H_t(\varepsilon_t) / (t \varepsilon_t) \to 0$
as $t \to \infty$,
then there is a limit-computable policy that is weakly asymptotically optimal
in the class of all computable stochastic environments.
\end{theorem}
\begin{proof}
Analogously to \autoref{thm:complexity-M} (i) we get that
$\xi$ is lower semicomputable,
and hence the optimal reward-seeking value function $V^*_\nu$
is limit computable~\cite[Lem.\ 21]{LH:2015computability}.
Hence by \autoref{lem:complexity-optimal-policies},
there is a limit-computable $2^{-t}$-optimal reward-seeking policy $\pi_\xi$
for the universal mixture $\xi$~\cite[Cor.\ 22]{LH:2015computability}.
By \autoref{thm:limit-computable-knowledge-seeking}
there are limit-computable $\epsilon_t/2$-optimal
information-seeking policies $\pi_I^t$ with lifetime $t + H_t(\varepsilon_t)$.
We define a policy $\pi$ analogously to {\BayesExp}
with $\pi_I^t$ and $\pi_\xi$ instead of the optimal policies:
\begin{align*}
&\text{If $V^*_I(\ae_{<t}) > \varepsilon_t$ for lifetime $t + H_t(\varepsilon_t)$,
then follow $\pi_I^t$ for $H_t(\varepsilon_t)$ steps.} \\
&\text{Otherwise, follow $\pi_\xi$ for one step.}
\end{align*}
Since $V^*_I$, $\pi_I$, and $\pi_\xi$ are limit computable,
the policy $\pi$ is limit computable.
Furthermore, $\pi_\xi$ is $2^{-t}$-optimal and $2^{-t} \to 0$,
so $V^{\pi_\xi}_\xi(\ae_{<t}) \to V^*_\xi(\ae_{<t})$
as $t \to \infty$.

Now we can proceed analogously to the proof of \cite[Thm.\ 5.6]{Lattimore:2013},
which consists of three parts.
First, 
it is shown that
the value of the $\xi$-optimal reward-seeking policy $\pi^*_\xi$
converges to the optimal value
for exploitation time steps (second branch in the definition of $\pi$)
in the sense that
$V^{\pi^*_\xi}_\mu \to V^*_\mu$.
This carries over to the $2^{-t}$-optimal policy $\pi_\xi$,
since the key property is that on exploitation steps,
$V^*_I < \varepsilon_t$;
i.e., $\pi$ only exploits if potential knowledge-seeking value is low.
In short, we get for exploitation steps
\[
    V^{\pi_\xi}_\xi(\ae_{<t})
\to V^{\pi^*_\xi}_\xi(\ae_{<t})
\to V^{\pi^*_\xi}_\mu(\ae_{<t})
\to V^*_\mu(\ae_{<t})
\text{ as } t \to \infty.
\]

Second, 
it is shown that the density of exploration steps vanishes.
This result carries over since the condition $V^*_I(\ae_{<t}) > \varepsilon_t$
that determines exploration steps is exactly the same as for {\BayesExp}
and $\pi^t_I$ is $\varepsilon_t/2$-optimal.

Third,
the results of part one and two are used to conclude that
$\pi$ is weakly asymptotically optimal.
This part carries over to our proof.
\qed
\end{proof}

\section{Summary}
\label{sec:summary}

When using Solomonoff's prior for induction,
we need to evaluate conditional probabilities.
We showed that
conditional $M$ and $M\norm$ are limit computable (\autoref{thm:complexity-M}),
and that $\MM$ and $\MM\norm$ are not limit computable
(\autoref{thm:MM-is-not-Delta2} and \autoref{cor:MMnorm-is-not-2});
see \autoref{tab:complexity-induction} on
\hyperref[tab:complexity-induction]{page~\pageref*{tab:complexity-induction}}.
This result implies that we can approximate $M$ or $M\norm$ for prediction,
but not the measure mixture $\MM$ or $\MM\norm$.

In some cases, normalized priors have advantages.
As illustrated in \autoref{ex:ksa-unnormalized},
unnormalized priors can make the entropy-seeking agent mistake
the entropy gained from the probability assigned to finite strings
for knowledge.
From $M\norm \geq M$ we get that
$M\norm$ predicts just as well as $M$, and
by \autoref{thm:complexity-M}
we can use $M\norm$ without losing limit computability.

Any method that tries to tackle the reinforcement learning problem
has to balance between exploration and exploitation.
AIXI strikes this balance in the Bayesian way.
However,
this does not lead to enough exploration~\cite{Orseau:2013,LH:2015priors}.
Our agent cares more about the present than the future---%
hence an investment in form of exploration is discouraged.
To counteract this,
we can add a knowledge-seeking component to the agent.
In \autoref{sec:complexity-knowledge-seeking} we discussed two variants
of knowledge-seeking agents:
entropy-seekers~\cite{Orseau:2014ksa} and
information-seekers~\cite{OLH:2013ksa}.
We showed that
$\varepsilon$-optimal knowledge-seeking agents are limit computable
and optimal knowledge-seeking agents are $\Delta^0_3$
(\autoref{thm:limit-computable-knowledge-seeking}).

We set out with the goal of finding a perfect reinforcement learning agent
that is limit computable.
Weakly asymptotically optimal agents can be considered a suitable candidate,
since they are currently the only known general reinforcement learning agents
which are optimal in an objective sense~\cite{LH:2015priors}.
We discussed Lattimore's {\BayesExp}~\cite[Ch.\ 5]{Lattimore:2013},
which relies
on Solomonoff induction to learn its environment and
on a knowledge-seeking component for extra exploration.
Our results culminated in
a limit-computable weakly asymptotically optimal agent
(\autoref{thm:wao-limit-computable}).
based on Lattimore's {\BayesExp}.
In this sense our goal has been achieved.

\paragraph{Acknowledgement.}
This work was supported by ARC grant DP150104590.
We thank Tom Sterkenburg for feedback on the proof of
\autoref{thm:M-conditional-is-not-Sigma1}.


\bibliographystyle{abbrv}
\bibliography{../ai}


\appendix
\section*{List of Notation}
\label{app:notation}

\begin{longtable}{lp{0.91\textwidth}}
$:=$
	& defined to be equal \\
$\mathbb{N}$
	& the natural numbers, starting with $0$ \\
$A, B$
	& sets of natural numbers \\
$\one_A$
	& the characteristic function
	that is $1$ if its argument is an element of the set $A$ and $0$ otherwise \\
$\X^*$
	& the set of all finite strings over the alphabet $\X$ \\
$\X^\infty$
	& the set of all infinite strings over the alphabet $\X$ \\
$\X^\sharp$
	& $\X^\sharp := \X^* \cup \X^\infty$,
	the set of all finite and infinite strings over the alphabet $\X$ \\
$x, y$
	& finite or infinite strings, $x, y \in \X^\sharp$ \\
$x \sqsubseteq y$
	& the string $x$ is a prefix of the string $y$ \\
$\epsilon$
	& the empty string, the history of length $0$ \\
$\varepsilon$
	& a small positive real number \\
$\A$
	& the (finite) set of possible actions \\
$\O$
	& the (finite) set of possible observations \\
$\E$
	& the (finite) set of possible percepts,
	$\E \subset \O \times \mathbb{R}$ \\
$M$
	& Solomonoff's prior defined in \eqref{eq:def-M} \\
$\MM$
	& the measure mixture defined in \eqref{eq:def-MM} \\
$\nu\norm$
	& Solomonoff normalization of the semimeasure $\nu$
	defined in \eqref{eq:normalization} \\
$\alpha, \beta$
	& two different actions, $\alpha, \beta \in \A$ \\
$a_t$
	& the action in time step $t$ \\
$e_t$
	& the percept in time step $t$ \\
$o_t$
	& the observation in time step $t$ \\
$r_t$
	& the reward in time step $t$, bounded between $0$ and $1$ \\
$\ae_{<t}$
	& the first $t - 1$ interactions,
	$a_1 e_1 a_2 e_2 \ldots a_{t-1} e_{t-1}$ \\
$\gamma$
	& the discount function $\gamma: \mathbb{N} \to \mathbb{R}_{\geq0}$ \\
$\Gamma_t$
	& a discount normalization factor,
	$\Gamma_t := \sum_{i=t}^\infty \gamma(i)$ \\
$H_t(\varepsilon)$
	& the effective horizon,
	$H_t(\varepsilon) = \min \{ H \mid \Gamma_{t + H} / \Gamma_t \leq \varepsilon \}$ \\
$\pi$
	& a policy, i.e., a function $\pi: \H \to \A$ \\
$V^\pi_H$
	& the entropy-seeking value of the policy $\pi$
	(see \autoref{def:V-entropy}) \\
$V^\pi_I$
	& the information-seeking value of the policy $\pi$
	(see \autoref{def:V-information}) \\
$V^\pi_\nu$
	& the reward-seeking value of policy $\pi$ in environment $\nu$
	(see \autoref{def:V-reward}) \\
$V^\pi$
	& the entropy-seeking/information-seeking/reward-seeking value
	of policy $\pi$ \\
$V^*$
	& the optimal entropy-seeking/information-seeking/reward-seeking value \\
$\phi$
	& a computable function \\
$S$ &
	a computable relation over natural numbers \\
$n, k, i$
	& natural numbers \\
$t$
	& (current) time step \\
$m$
	& lifetime of the agent (a function of the current time step $t$) \\
$\M$
	& the class of all lower semicomputable chronological conditional semimeasures;
	our environment class \\
$\nu$
	& lower semicomputable semimeasure \\
$\mu$
	& computable measure, the true environment \\
$\xi$
	& the universal mixture over all environments in $\M$ \\
\end{longtable}

\section*{Open Questions}
\label{app:open-questions}

\begin{enumerate}[1.]
\item Can the upper bound of $\Delta^0_3$ for knowledge-seeking policies
	be improved?
\item Is {\BayesExp} limit computable?
\item Does the lower given in \autoref{thm:MM-is-not-Delta2}
	and \autoref{cor:MMnorm-is-not-2} hold for any universal Turing machine?
\end{enumerate}
We expect the answers to questions 1 and 2 to be negative
and the answer to question 3 to be positive.

\end{document}